\def\year{2019}\relax
\documentclass[letterpaper]{article} 
\usepackage{aaai19}  
\usepackage{times}  
\usepackage{helvet}  
\usepackage{courier}  
\usepackage{url}  
\usepackage{graphicx}  

\usepackage{booktabs}       
\usepackage{amsfonts}       
\usepackage{amssymb}
\usepackage{amsmath}
\usepackage{amsthm}
\usepackage{algorithm}
\usepackage{algorithmic}
\usepackage{subfig}
\usepackage{graphicx}
\usepackage{enumerate}
\usepackage{diagbox}

\newtheorem{definition}{Definition}
\newtheorem{theorem}{Theorem}
\newtheorem{lemma}{Lemma}
\newtheorem{proposition}{Proposition}

\newtheorem{claim}{Claim}

\begin{document}
%
\title{On Geometric Alignment in Low Doubling Dimension}
\author{Hu Ding\\
School of Computer Science and Technology\\
School of Data Science\\
University of Science and Technology of China\\
He Fei, China, 230026\\
\And Mingquan Ye\\
Department of Computer Science and Engineering\\
Michigan State University\\
East Lansing, MI, USA, 48824}

\maketitle
\begin{abstract}
In real-world, many problems can be formulated as the alignment between two geometric patterns. Previously, a great amount of research focus on the alignment of 2D or 3D patterns, especially in the field of computer vision. Recently, the alignment of geometric patterns in high dimension finds several novel applications, and has attracted more and more attentions. However, the research is still rather limited in terms of algorithms. To the best of our knowledge, most existing approaches for high dimensional alignment are just simple extensions of their counterparts for 2D and 3D cases, and often suffer from the issues such as high complexities. In this paper, we propose an effective framework to compress the high dimensional geometric patterns and approximately preserve the alignment quality. As a consequence, existing alignment approach can be applied to the compressed geometric patterns and thus the time complexity is significantly reduced. Our idea is inspired by the observation that high dimensional data often has a low intrinsic dimension. We adopt the widely used notion ``doubling dimension'' to measure the extents of our compression and the resulting approximation. Finally, we test our method on both random and real datasets; the experimental results reveal that running the alignment algorithm on compressed patterns can achieve similar qualities, comparing with the results on the original patterns, but the running times (including the times cost for compression) are substantially lower.  \end{abstract}

\section{Introduction}
\label{sec-intro}

Given two geometric patterns, the problem of alignment is to find their appropriate spatial positions so as to minimize the difference between them. In general, a geometric pattern is represented by a set of (weighted) points in the space, and their difference is often measured by some objective function. In particular, geometric alignment finds many applications in the field of computer vision, such as image retrieval, pattern recognition, fingerprint and facial shape alignment, etc~\cite{cohen1999earth,maltoni2009handbook,cao2014face}. For different applications, we may have different constraints for the alignment, e.g., we allow rigid transformations for fingerprint alignment. In addition, Earth Mover's Distance (EMD)~\cite{rubner2000earth} has been widely adopted as the metric for measuring the difference of patterns in computer vision, where its major advantage over other measures is the robustness with respect to noise in practice. 
%
%
%
Besides the computer vision applications in 2D or 3D, recent research shows that a number of high dimensional problems can be solved by geometric alignments. We briefly introduce several interesting examples below. 

\textbf{(1)} The research on natural language processing has revealed that different languages often share some similar structure at the word level~\cite{youn2016universal}; in particular, the recent study on word semantics embedding has also shown the existence of structural isomorphism across languages~\cite{mikolov2013exploiting}, and further finds that EMD can serve as a good distance for languages or documents~\cite{DBLP:conf/emnlp/ZhangLLS17,kusner2015word}. Therefore, ~\cite{DBLP:conf/emnlp/ZhangLLS17} proposed to learn the transformation between different languages without any cross-lingual supervision, and the problem is reduced to minimizing the EMD via finding the optimal  geometric alignment in high dimension.  
\textbf{(2)} A Protein-Protein Interaction (PPI) network is a graph representing the interactions among proteins. Given two PPI networks, finding their alignment is a fundamental bioinformatics problem for understanding the correspondences between different species~\cite{malod2017unified}. However, most existing approaches require to solve the NP-hard subgraph isomorphism problem and often suffer from high computational complexities. To resolve this issue, ~\cite{DBLP:conf/aaai/LiuDC017} recently applied the geometric embedding techniques to develop a new framework based on geometric alignment in Euclidean space. 
\textbf{(3)} 
Other applications of high dimensional alignment include domain adaptation and indoor localization. Domain adaptation is an important problem in machine learning, where the goal is to predict the annotations of a given unlabeled dataset by determining the transformation (in the form of EMD) from a labeled dataset~\cite{pan2010survey}; as the rapid development of wireless technology, indoor localization becomes rather important for locating a person or device inside a building. 
Recent studies show that they both can be well formulated as geometric alignments in high dimension. We refer the reader to \cite{courty2017joint} and \cite{yang2012locating} for more details.

Despite of the above studies in terms of applications, the research on the algorithms is still rather limited and far from being satisfactory. Basically, we need to take into account of the high dimensionality and large number of points of the geometric patterns, simultaneously. 
In particular, as the developing of data acquisition techniques, data sizes increase fast and designing efficient algorithms for high dimensional alignment will be important for some applications. For example, due to the rapid progress of high-throughput sequencing technologies, biological data are growing exponentially~\cite{yin2017computing}. 
In fact, even for the 2D and 3D cases, solving the geometric alignment problem is quite challenging yet. For example, the well-known {\em iterative closest point (ICP)} method~\cite{besl1992method} can only guarantee to obtain a local optimum. 
More previous works will be discussed in Section~\ref{sec-prior}.



To the best of our knowledge, we are the first to consider developing efficient algorithmic framework for the geometric alignment problem in large-scale and high dimension. Our idea is inspired by the observation that many real-world datasets often manifest low intrinsic dimensions~\cite{belkin2003problems}. 
For example, human handwriting images can be well embedded to some low dimensional manifold though the Euclidean dimension can be very high~\cite{tenenbaum2000global}. Following this observation, we consider to exploit the widely used notion, ``doubling dimension''~\cite{krauthgamer2004navigating,talwar2004bypassing,karger2002finding,har2006fast,dasgupta2013randomized}, to deal with the geometric alignment problem. 
Doubling dimension is particularly suitable to depict the data having low intrinsic dimension. We prove that the given geometric patterns with low doubling dimensions can be substantially compressed so as to save a large amount of running time when computing the alignment. More importantly, our compression approach is an independent step, and hence can serve as the preprocessing for various alignment methods.


The rest of the paper is organized as follows. We provide the definitions that are used throughout this paper in Section~\ref{sec-pre}, and discuss some existing approaches and our main idea in Section~\ref{sec-prior}. Then, we present our algorithm, analysis, and the time complexity in detail in Section~\ref{sec-aa} and \ref{sec-time}. Finally, we study the practical performance of our proposed algorithm in Section~\ref{sec-exp}.


\subsection{Preliminaries}
\label{sec-pre}

Before introducing the formal definition of geometric alignment, we need to define EMD and rigid transformation first.

\begin{definition}[Earth Mover's Distance (EMD)]
\label{def-emd}
Let $A=\{a_1, a_2, \cdots, a_{n_1}\}$ and $B=\{b_1, b_2, \cdots, b_{n_2}\}$ be two sets of weighted points in $\mathbb{R}^d$ with nonnegative weights $\alpha_i$ and $\beta_j$ for each $a_i\in A$ and $b_j\in B$ respectively, and $W_A$ and $W_B$ be their respective total weights. The earth mover's distance between $A$ and $B$ is $\mathcal{EMD}(A, B)=$
\begin{eqnarray}
\frac{1}{\min\{W_A, W_B\}}\min_{F}\sum^{n_1}_{i=1}\sum^{n_2}_{j=1}f_{ij}||a_i-b_j||^2, \label{for-emd}
\end{eqnarray} 
where $F=\{f_{ij}\}$ is a feasible flow from $A$ to $B$, i.e., each $f_{ij}\geq 0$, $\sum^{n_1}_{i=1}f_{ij}\leq\beta_j$, $\sum^{n_2}_{j=1}f_{ij}\leq\alpha_i$, and $\sum^{n_1}_{i=1}\sum^{n_2}_{j=1}f_{ij}=\min\{W_A, W_B\}$. 
\end{definition}
\vspace{-0.05in}

Intuitively, EMD can be viewed as the minimum transportation cost between $A$ and $B$, where the weights of $A$ and $B$ are the ``supplies'' and ``demands'' respectively, and the cost of each edge connecting a pair of points from $A$ to $B$ is their ``ground distance''. In general, the ``ground distance'' can be defined in various forms, and here we simply use the squared distance because it is widely adopted in practice. 

\begin{definition}[Rigid Transformation]
\label{def-rt}
Let $P$ be a set of points in $\mathbb{R}^{d}$. A rigid transformation $\mathcal{T}$ on $P$ is a 
 transformation ({\em i.e.,} rotation, translation, reflection, or their combination) which preserves the pairwise distances of the points in $P$. 
\end{definition}

%


We consider rigid transformation for alignment, because it is very natural to interpret in real-world and has already been used by the aforementioned applications.

\begin{definition}[Geometric Alignment]
\label{def-align}
Given two weighted point sets $A$ and $B$ as described in Definition~\ref{def-emd}, the problem of geometric alignment between $A$ and $B$ under rigid transformation is to determine a rigid transformation $\mathcal{T}$ for $B$ so as to minimize the earth mover's distance  $\mathcal{EMD}(A, \mathcal{T}(B))$.
\end{definition}


As previously mentioned, we consider to use doubling dimension to describe high dimensional data having low intrinsic dimension. We denote a metric space by $(X, d_X)$ where $d_X$ is the distance function of the set $X$. For instance, we can imagine that $X$ is a set of points in a low dimensional manifold and $d_X$ is simply Euclidean distance. For any $x\in X$ and $r\geq 0$, $Ball(x, r)=\{p\in X\mid d_X(x,p)\leq r\}$ indicates the ball of radius $r$ around $x$ (note that $Ball(x, r)$ is a subset of $X$).

 \begin{definition}[Doubling Dimension]
\label{def-dd}
The doubling dimension of a metric space $(X, d_X)$ is the smallest number $\rho$, such that for any $x\in X$ and $r\geq 0$, $Ball(x, 2r)$ is always covered by the union of at most $2^\rho$ balls with radius $r$.
\end{definition}

Doubling dimension describes the expansion rate of $(X, d_X)$; intuitively, we can imagine a set of points uniformly distributed inside a $\rho$-dimensional hypercube, where its doubling dimension is $O(\rho)$ but the Euclidean dimension can be very high. For a more general case, a manifold in high dimensional Euclidean space may have a very low doubling dimension, as many examples studied in machine learning~\cite{belkin2003problems}. 
Unfortunately, as shown before~\cite{laakso2002plane}, such low doubling dimensional metrics cannot always be embedded to low dimensional Euclidean spaces with low distortion in terms of Euclidean distance. 
Therefore, we need to design the techniques being able to manipulate the data in high dimensional Euclidean space directly.

\subsection{Existing Results and Our Approach}
\label{sec-prior}


%
%


If building a bipartite graph, where the two columns of vertices correspond to the points of $A$ and $B$ respectively and each edge connecting $(a_i, b_j)$ has the weight $||a_i-b_j||^2$, we can see that computing EMD actually is a min-cost flow problem. Many min-cost flow algorithms have been developed in the past decades, such as 
%
{\em Network simplex algorithm}~\cite{ahuja1993network}, a specialized version of the simplex algorithm.
%
Since EMD is an instance of min-cost flow problem in Euclidean space and the geometric techniques (e.g., the geometric sparse spanner) are applicable, a number of faster algorithms have been proposed in the area of computational geometry~\cite{DBLP:conf/compgeom/AgarwalFPVX17,cabello2008matching,indyk2007near}, however, most of them only work for low dimensional case. 
Several EMD algorithms with assumptions (or some modifications on the objective function of EMD) have been studied in practical areas~\cite{pele2009fast,ling2007efficient,DBLP:journals/siamsc/BenamouCCNP15}.

Computing the geometric alignment of $A$ and $B$ is more challenging, since we need to determine the rigid transformation and EMD flow simultaneously. Moreover, due to the flexibility of rigid transformations, we cannot apply the EMD embedding techniques~\cite{IT03,andoni2009efficient} to relieve the challenges. For example,  the embedding can only preserve the EMD between $A$ and $B$; however, since there are infinite number of possible rigid transformations $\mathcal{T}$ for $B$ (note that we do not know $\mathcal{T}$ in advance), it is difficult to also preserve the EMD between $A$ and $\mathcal{T}(B)$. In theory, 
\cite{cabello2008matching} presented a $(2+\epsilon)$-approximation solution for the 2D case, and later \cite{klein2005approximation} achieved an $O(2^{d-1})$-approximation in $\mathbb{R}^d$; 
\cite{DBLP:journals/algorithmica/DingX17} proposed a PTAS for constant dimension. However, these theoretical algorithms cannot be efficiently implemented when the dimensionality is not constant. 
\cite{DBLP:journals/algorithmica/DingX17} also mentioned that any constant factor approximation needs a time complexity at least $n^{\Omega(d)}$ based on some reasonable assumption in the theory of computational complexity, where $n=\max\{|A|, |B|\}$. That is, it is unlikely to obtain a $(1+\epsilon)$-approximation within a practical running time, especially when $n$ is very large.

In practice, \cite{cohen1999earth} proposed an alternating minimization approach for computing the geometric alignment of $A$ and $B$. 
Several other approaches~\cite{cornea20053d,todorovic2008region} based on graph matching are inappropriate to be extended for high dimensional alignment. 
In machine learning, a related topic is called ``manifold alignment''~\cite{ham2005semisupervised,wang2011manifold}; however, it usually has different settings and applications, and thus is out of the scope of this paper. 

Because the approach of~\cite{cohen1999earth}  is closely related to our proposed algorithm, we introduce it with more details for the sake of completeness. 
Roughly speaking, their approach is similar to the {\em Iterative Closest Point method} (ICP) method~\cite{besl1992method}, where its each iteration alternatively updates the EMD flow and rigid transformation. Thus it converges to some local optimum eventually. 
To update the rigid transformation, we can apply {\em Orthogonal Procrustes (OP) analysis}~\cite{schonemann1966generalized}. The original OP analysis is only for unweighted point sets, and hence we need some significant modification for our problem.


Suppose that the EMD flow $F=\{f_{ij}\}$ is fixed and the rigid transformation is waiting to update in the current stage. We imagine two new sets of weighted points  
\begin{eqnarray}
\hat{A}&=&\{a^1_1, a^2_1, \cdots, a^{n_2}_1; a^1_2, a^2_2, \cdots, a^{n_2}_2;\nonumber\\
&& \cdots; a^1_{n_1}, a^2_{n_1}, \cdots, a^{n_2}_{n_1}\};\label{for-cg1}\\
\hat{B}&=&\{b^1_1, b^1_2, \cdots, b^1_{n_2}; b^2_1, b^2_2, \cdots, b^2_{n_2};\nonumber\\
&& \cdots; b^{n_1}_1, b^{n_1}_{2}, \cdots, b^{n_1}_{n_2}\},\label{for-cg2}
\end{eqnarray}
where each $a^j_i$ (resp., $b^i_j$) has the weight $f_{ij}$ and the same spatial position of  $a_i$ (resp., $b_j$). With a slight abuse of notations, we also use $a^j_i$ (resp., $b^i_j$) to denote the corresponding $d$-dimensional column vector in the following description.
First, we take a translation vector $\overrightarrow{v}$ such that the weighted mean points of $\hat{A}$ and $\hat{B}+\overrightarrow{v}$ coincide with each other (this can be easily proved, due to the fact that the objective function uses squared distance~\cite{cohen1999earth}). Second, by OP analysis, we compute an orthogonal matrix $\mathcal{R}$ for $\hat{B}+\overrightarrow{v}$ to minimize its weighted $L^2_2$ difference to $\hat{A}$. For this purpose, we generate two $d\times (n_1 n_2)$ matrices $M_A$ and $M_B$, where each point of $\hat{A}$ (resp., $\hat{B}+\overrightarrow{v}$) corresponds to an individual column of $M_A$ (resp., $M_B$); for example, a point $a^j_i\in \hat{A}$ (resp., $b^i_j+\overrightarrow{v}\in \hat{B}+\overrightarrow{v}$) corresponds to a column $\sqrt{f_{ij}}a^j_i$ (resp., $\sqrt{f_{ij}}(b^i_j+\overrightarrow{v})$) in $M_A$ (resp., $M_B$). Let  the SVD of $M_A\times M^T_B$ be $U\Sigma V^T$, and the optimal orthogonal matrix $\mathcal{R}$ should be $UV^T$ through OP analysis. Actually we do not need to really construct the large matrices $M_A$ and $M_B$, since many of the columns are identical.  Instead, we can compute the multiplication $M_A\times M^T_B$ in $O(n_1n_2d+\min\{n_1, n_2\}\cdot d^2)$ time (see Lemma~\ref{lem-app} in Appendix). Therefore, the time complexity for obtaining the optimal $\mathcal{R}$ is $O(n_1n_2d+\min\{n_1, n_2\}\cdot d^2+d^3)$.

\begin{proposition}
\label{pro-cg}
Each iteration of the approach of~\cite{cohen1999earth} takes $\Gamma(n_1, n_2, d)+O(n_1n_2d+\min\{n_1, n_2\}\cdot d^2+d^3)$ time, where $\Gamma(n_1, n_2, d)$ indicates the time complexity of the EMD algorithm it adopts. In practice, we usually assume $n_1, n_2=O(n)$ with some $n\geq d$, and then the complexity can be simplified to be $\Gamma(n, d)+O(n^2d)$.
\end{proposition}





The bottleneck is that the algorithm needs to repeatedly compute the EMD and transformation, especially when $n$ and $d$ are large (usually $\Gamma(n,d)=\Omega(n^2 d)$). Based on the property of low doubling dimension, we construct a pair of compressed point sets to replace the original $A$ and $B$, and run the same algorithm on the compressed data instead. As a consequence, the running time is reduced significantly. Note that our compression step is \textbf{independent of} the approach~\cite{cohen1999earth}; actually, any alignment method with the same objective function in Definition~\ref{def-align} can benefit from our compression idea. 
Recently, \cite{DBLP:journals/corr/NasserJF15} proposed a core-set based compression approach to speed up the computation of alignment. However,  their method requires to know the correspondences between the point sets in advance and therefore it is not suitable to handle EMD; moreover, their compression achieves the advantage only when $d$ is small.


\section{The Algorithm and Analysis}
\label{sec-aa}
Our idea starts from the widely studied 
$k$-center clustering.
 Given an integer $k\geq 1$ and a point set $P$ in some metric space, $k$-center clustering is to partition $P$ into $k$ clusters and cover each cluster by an individual ball, such that the maximum radius of the balls is minimized. \cite{gonzalez1985clustering} presented an elegant $2$-approximation algorithm, where the radius of each resulting ball (i.e., cluster) is at most two times the optimum. Initially, it selects an arbitrary point, say $c_1$, from the input $P$ and sets $S=\{c_1\}$; then it iteratively selects a new point which has the largest distance to $S$ among the points of $P$ and adds it to $S$, until $|S|=k$ (the distance between a point $q$ and $S$ is defined as $\min\{||q-p||\mid p\in S\}$); suppose $S=\{c_1, \cdots, c_k\}$, and then $P$ is covered by the $k$ balls $Ball(c_1, r), \cdots, Ball(c_k, r)$ with  $r\leq\min\{||c_i-c_j||\mid 1\leq i\neq j\leq k\}$. It is able to prove that $r$ is at most two times the optimal radius of the given instance. 
Using the property of doubling dimension, we have the following lemma.

\begin{lemma}
\label{lem-number}
Let $P$ be a point set in $ \mathbb{R}^d$ with the doubling dimension $\rho\ll d$. The diameter of $P$ is denoted by $\Delta$, i.e., $\Delta=\max\{||p-q||\mid p, q\in P\}$. Given a small parameter $\epsilon>0$, if one runs the $k$-center clustering algorithm of Gonzalez by setting $k=(\frac{2}{\epsilon})^\rho$, the radius of each resulting ball is at most $\epsilon \Delta$. 
\end{lemma}
\begin{proof}
Let $S$ be the set of $k$ points by Gonzalez's algorithm, and the resulting radius be $r$. We also define the aspect ratio of $S$ as the ratio between the maximum and minimum pairwise distances in $S$. Then, it is easy to see that the aspect ratio of $S$ is at most $\Delta/r$. Now, we need the following Claim~\ref{claim-compress} from~\cite{krauthgamer2004navigating,talwar2004bypassing}. Actually, the claim can be obtained by recursively applying the definition of doubling dimension.


\begin{claim}
\label{claim-compress}
Let $(X, d_X)$ be a metric space with the doubling dimension $\rho$, and $Y\subset X$. If the aspect ratio of $Y$ is upper bounded by some positive value $\alpha$, then $|Y|\leq 2^{\rho\lceil\log_2 \alpha\rceil}$.
\end{claim}

Replacing $X$ and $Y$ by $P$ and $S$ respectively in the above claim, we have
\begin{eqnarray}
|S|\leq 2^{\rho\lceil\log_2\Delta/r\rceil}\leq 2^{\rho (1+\log_2\Delta/r)}. \label{for-number1} 
\end{eqnarray}
Since $|S|=(2/\epsilon)^\rho$, (\ref{for-number1}) implies $r\leq \epsilon \Delta$.
\end{proof}

Let $A$ and $B$ be the two given point sets in Definition~\ref{def-align}, and the maximum of their diameters be $\Delta$. We also assume that they both have the doubling dimension at most $\rho$. Our idea for compressing $A$ and $B$ is as follows. As described in Lemma~\ref{lem-number}, we set $k=(\frac{2}{\epsilon})^\rho$ and run Gonzalez's algorithm on $A$ and $B$ respectively. We denote by $S_A=\{c^A_1,\cdots, c^A_k\}$ and $S_B=\{c^B_1,\cdots, c^B_k\}$ the obtained sets of $k$-cluster centers. For each cluster center $c^A_j$ (resp., $c^B_j$), we assign a weight that is equal to the total weights of the points in the corresponding cluster. As a consequence, we obtain a new instance $(S_A, S_B)$ for geometric alignment. It is easy to know that the total weights of $S_A$ (resp., $S_B$) is equal to $W_A$ (resp., $W_B$). The following theorem shows that we can achieve an approximate solution for the instance $(A, B)$ by solving the alignment of $(S_A, S_B)$.


\begin{theorem}
\label{the-quality}
Suppose $\epsilon>0$ is a small parameter in Lemma~\ref{lem-number}. Given any $c\geq 1$, let $\tilde{\mathcal{T}}$ be a rigid transformation yielding $c$-approximation for minimizing $\mathcal{EMD}\big(S_A, \mathcal{T}(S_B)\big)$ in Definition~\ref{def-align}. Then, $\mathcal{EMD}\big(A, \tilde{\mathcal{T}}(B)\big)$
\begin{eqnarray}
&\leq& c(1+2\epsilon)^2\cdot\min_\mathcal{T}\mathcal{EMD}\big(A, \mathcal{T}(B)\big)\nonumber\\
&&+2\epsilon(c+1+2c\epsilon)(1+2\epsilon)\Delta^2\nonumber\\
&=& c\big(1+O(\epsilon)\big)\cdot\min_\mathcal{T}\mathcal{EMD}\big(A, \mathcal{T}(B)\big)\nonumber\\
&&+2\epsilon \big(1+O(\epsilon)\big)(c+1)\Delta^2. \label{for-quality1}
\end{eqnarray}
\end{theorem}
\begin{proof}
First, we denote by $\mathcal{T}_{opt}$ the optimal rigid transformation achieving $\min_\mathcal{T}\mathcal{EMD}\big(A, \mathcal{T}(B)\big)$. Since $\tilde{\mathcal{T}}$ yields $c$-approximation for minimizing $\mathcal{EMD}\big(S_A, \mathcal{T}(S_B)\big)$, we have $\mathcal{EMD}\big(S_A, \tilde{\mathcal{T}}(S_B)\big)$
\begin{eqnarray}
&\leq& c\cdot \min_\mathcal{T}\mathcal{EMD}\big(S_A, \mathcal{T}(S_B)\big)\nonumber\\
&\leq& c\cdot\mathcal{EMD}\big(S_A, \mathcal{T}_{opt}(S_B)\big).\label{for-quality3}
\end{eqnarray}

Recall that each point $c^A_j$ (resp., $c^B_j$) has the weight equal to the total weights of the points in the corresponding cluster. For instance, if the cluster contains $\{a_{j(1)}, a_{j(2)}, \cdots, a_{j(h)}\}$, the weight of $c^A_j$ should be $\sum^h_{l=1}\alpha_{j(l)}$; actually, we can view $c^A_j$ as $h$ overlapping points $\{a'_{j(1)}, a'_{j(2)}, \cdots, a'_{j(h)}\}$ with each $a'_{j(l)}$ having the weight $\alpha_{j(l)}$. 
Therefore, for the sake of convenience, we use another representation for $S_A$ and $S_B$ in our proof below:
\begin{eqnarray}
S_A=\{a'_1, \cdots, a'_{n_1}\}\text{ and } S_B=\{b'_1, \cdots, b'_{n_2}\}, \label{for-newrep}
\end{eqnarray}
where each $a'_j$ (resp., $b'_j$) has the weight $\alpha_j$ (resp., $\beta_j$). Note that $S_A$ and $S_B$ only have $k$ distinct positions respectively in the space. Moreover, due to Lemma~\ref{lem-number}, we know that $||a'_i-a_i||$, $||b'_j-b_j||\leq \epsilon\Delta$ for any $1\leq i\leq n_1$ and $1\leq j\leq n_2$, and these bounds are invariant under any rigid transformation in the space. Consequently, for any pair $(i, j)$ and rigid transformation $\mathcal{T}$, we have $||a_i-\mathcal{T}(b_j)||^2$
\begin{eqnarray}
&\leq& \big(||a_i-a'_i||+||a'_i-\mathcal{T}(b'_j)||\nonumber\\
&&+||\mathcal{T}(b'_j)-\mathcal{T}(b_j)||\big)^2\nonumber\\
&\leq&\big(||a'_i-\mathcal{T}(b'_j)||+2\epsilon\Delta\big)^2\nonumber\\
&=&||a'_i-\mathcal{T}(b'_j)||^2+4\epsilon\Delta||a'_i-\mathcal{T}(b'_j)||+4\epsilon^2\Delta^2\nonumber\\
&\leq&||a'_i-\mathcal{T}(b'_j)||^2+2\epsilon\big(\Delta^2+||a'_i-\mathcal{T}(b'_j)||^2\big)\nonumber\\
&&+4\epsilon^2\Delta^2 \nonumber\\
&=&(1+2\epsilon)||a'_i-\mathcal{T}(b'_j)||^2+(2\epsilon+4\epsilon^2)\Delta^2\label{for-quality4}
\end{eqnarray}
by applying triangle inequality. Using exactly the same idea, we have $||a'_i-\mathcal{T}(b'_j)||^2$
\begin{eqnarray}
\leq (1+2\epsilon)||a_i-\mathcal{T}(b_j)||^2+(2\epsilon+4\epsilon^2)\Delta^2.\label{for-quality5}
\end{eqnarray}
Based on Definition~\ref{def-emd}, we denote by $\tilde{F}=\{\tilde{f}_{ij}\}$ the induced flow of $\mathcal{EMD}\big(S_A, \tilde{\mathcal{T}}(S_B)\big)$ (using the representations (\ref{for-newrep}) for $S_A$ and $S_B$). Then (\ref{for-quality4}) directly implies that $\mathcal{EMD}\big(A, \tilde{\mathcal{T}}(B)\big)$
\begin{eqnarray}
&\leq&\frac{1}{\min\{W_A, W_B\}}\sum^{n_1}_{i=1}\sum^{n_2}_{j=1}\tilde{f}_{ij}||a_{i}-\tilde{\mathcal{T}}(b_j)||^2\nonumber\\
&\leq&\frac{1+2\epsilon}{\min\{W_A, W_B\}}\sum^{n_1}_{i=1}\sum^{n_2}_{j=1}\tilde{f}_{ij}||a'_{i}-\tilde{\mathcal{T}}(b'_j)||^2\nonumber\\
&&+(2\epsilon+4\epsilon^2)\Delta^2\nonumber\\
&=&(1+2\epsilon)\mathcal{EMD}(S_A, \tilde{\mathcal{T}}(S_B))\nonumber\\
&&+(2\epsilon+4\epsilon^2)\Delta^2.\label{for-quality6}
\end{eqnarray}
By the similar idea (replacing $\tilde{\mathcal{T}}$ by $\mathcal{T}_{opt}$, and exchanging the roles of $(A, B)$ and $(S_A, S_B)$),  (\ref{for-quality5}) directly implies that $\mathcal{EMD}\big(S_A, \mathcal{T}_{opt}(S_B)\big)$
\begin{eqnarray}
&\leq&(1+2\epsilon)\mathcal{EMD}\big(A, \mathcal{T}_{opt}(B)\big)\nonumber\\
&&+(2\epsilon+4\epsilon^2)\Delta^2.\label{for-quality7}
\end{eqnarray}
Combining (\ref{for-quality3}), (\ref{for-quality6}), and (\ref{for-quality7}), we have $\mathcal{EMD}\big(A, \tilde{\mathcal{T}}(B)\big)$
\begin{eqnarray}
&\leq&(1+2\epsilon)\mathcal{EMD}\big(S_A, \tilde{\mathcal{T}}(S_B)\big)\nonumber\\
&&+(2\epsilon+4\epsilon^2)\Delta^2\nonumber\\
&\leq&(1+2\epsilon)\cdot c\cdot\mathcal{EMD}\big(S_A, \mathcal{T}_{opt}(S_B)\big)\nonumber\\
&&+(2\epsilon+4\epsilon^2)\Delta^2 \nonumber\\
&\leq& c(1+2\epsilon)^2\cdot\mathcal{EMD}\big(A, \mathcal{T}_{opt}(B)\big)\nonumber\\
&&+2\epsilon(c+1+2c\epsilon)(1+2\epsilon)\Delta^2,
\end{eqnarray}
and the proof is completed.
\end{proof}
 
When $\epsilon$ is small enough, Theorem~\ref{the-quality} shows that $\mathcal{EMD}\big(A, \tilde{\mathcal{T}}(B)\big)\approx c\cdot\mathcal{EMD}\big(A, \mathcal{T}_{opt}(B)\big)$. That is, $\tilde{\mathcal{T}}$, the solution of $(S_A, S_B)$, achieves roughly the same performance on $(A, B)$. 
Consequently, we propose the approximation algorithm for geometric alignment (see Algorithm~\ref{alg-1}). We would like to emphasize that though we use the algorithm from~\cite{cohen1999earth} in Step 3, Theorem~\ref{the-quality} is an independent result; that is, any alignment method with the same objective function in Definition~\ref{def-align} can benefit from Theorem~\ref{the-quality}.

\renewcommand{\algorithmicrequire}{\textbf{Input:}}
\renewcommand{\algorithmicensure}{\textbf{Output:}}
\begin{algorithm}
   \caption{Geometric Alignment}
   \label{alg-1}
\begin{algorithmic}[1]
\REQUIRE An instance $(A, B)$ of the geometric alignment problem in Definition~\ref{def-align} with bounded doubling dimension $\rho$ in $\mathbb{R}^d$; $\epsilon\in(0,1)$.
\ENSURE A rigid transformation $\mathcal{T}$ of $B$ and the EMD flow between $A$ and $\mathcal{T}(B)$.
\STATE Set $k=(2/\epsilon)^\rho$.
\STATE Run Gonzalez's $k$-center clustering algorithm on $A$ and $B$, and obtain the sets of cluster centers $S_A$ and $S_B$ respectively.
\STATE Apply the existing alignment algorithm, e.g., ~\cite{cohen1999earth}, on $(S_A, S_B)$.
\STATE Obtain the rigid transformation $\mathcal{T}$ from Step 3, and compute the corresponding EMD flow between $A$ and $\mathcal{T}(B)$.
\end{algorithmic}
\end{algorithm}

\section{The Time Complexity}
\label{sec-time}
We analyze the time complexity of Algorithm~\ref{alg-1} and consider step 2-4 separately. 
To simplify our description, we use $n$ to denote $\max\{n_1, n_2\}$. In step 3, we suppose that the iterative approach~\cite{cohen1999earth} takes $\lambda\geq1$ rounds.

\textbf{Step 2.} A straightforward implementation of Gonzalez's  algorithm is selecting the $k$ cluster centers iteratively where the resulting running time is $O(knd)$. Several faster implementations for the high dimensional case with low doubling dimension have been studied before; their idea is to maintain some data structures to reduce the amortized complexity of each iteration. We refer the reader to~\cite{har2006fast} for more details.


\textbf{Step 3.} Since we run the algorithm~\cite{cohen1999earth} on the smaller instance $(S_A, S_B)$ instead of $(A, B)$, we know that the complexity of Step 3 is $O\Big(\lambda\big(\Gamma(k,d)+k^2 d+kd^2+d^3\big)\Big)$ by Proposition~\ref{pro-cg}.

\textbf{Step 4.} We need to compute the transformed $\mathcal{T}(B)$ first and then $\mathcal{EMD}(A, \mathcal{T}(B))$. Note that the transformation $\mathcal{T}$ is not off-the-shelf, because it is the combination of a sequence of rigid transformations from the iterative approach~\cite{cohen1999earth} in Step 3. Since it takes $\lambda$ rounds, $\mathcal{T}$ should be the multiplication of $\lambda$ rigid transformations. We use $(\mathcal{R}_l, \overrightarrow{v}_l)$ to denote the orthogonal matrix and translation vector obtained in the $l$-th round for $1\leq l\leq \lambda$. We can update $B$ round by round: starting from $l=1$, update $B$ to be $\mathcal{R}_l B+\overrightarrow{v}_l$ in each round; the whole time complexity will be $O(\lambda n d^2)$. In fact, we have a more efficient way by computing $\mathcal{T}$ first before transforming $B$.

\begin{lemma}
\label{lem-tran}
Let $(\mathcal{R}, \overrightarrow{v})$ be the orthogonal matrix and translation vector of $\mathcal{T}$. Then
\begin{eqnarray}
\mathcal{R}&=&\Pi^\lambda_{l=1}\mathcal{R}_l, \nonumber\\
\overrightarrow{v}&=&(\Pi^\lambda_{l=2}\mathcal{R}_l)\overrightarrow{v}_1+(\Pi^\lambda_{l=3}\mathcal{R}_l)\overrightarrow{v}_2+\nonumber\\
&&\cdots+\mathcal{R}_\lambda\overrightarrow{v}_{\lambda-1}+\overrightarrow{v}_\lambda,\label{for-tran}
\end{eqnarray}
and $\mathcal{T}(B)$ can be obtained in $O(\lambda d^3+n d^2)$ time. 
\end{lemma} 
\begin{proof}
The equations (\ref{for-tran}) can be easily verified by simple calculations. In addition, we can recursively compute the multiplications $\Pi^\lambda_{l=i}\mathcal{R}_l$ for $i=\lambda, \lambda-1, \cdots, 1$. Consequently, the orthogonal matrix $\mathcal{R}$ and translation vector $\overrightarrow{v}$ can be obtained in $O(\lambda d^3)$ time. In addition, the complexity for computing $\mathcal{T}(B)=\mathcal{R}B+\overrightarrow{v}$ is $O(n d^2)$.
\end{proof}

\begin{table*}[htbp]
\centering
\caption{Random dataset: EMDs and running times of different compression levels.}
\begin{tabular}{|l|c|c|c|c|c|c|c|}
\hline
$\gamma$  &$\textbf{1/50}$&$1/40$ &$1/30$&$1/20$ &$1/10$&$\textbf{1}$ \\ \hline
EMD&$0.948$&$0.946$&$0.945$&$0.943$&$0.941$&$0.933$  \\ \hline
Time (s)&$48.7$&$54.2$&$61.0$&$80.6$ &$144.6$ &$1418.2$ \\ \hline
\end{tabular}
\label{tab-random}
\end{table*}

\begin{table*}[htbp]
\centering
\caption{Linguistic datasets: EMDs and running times of different compression levels.
}
\begin{tabular}{|c|c|c|c|c|c|c|c|c|}
\hline
\diagbox[height=0.5cm,width=3.5cm]{Datasets}{$\gamma$ } &$\textbf{1/50}$ &$1/40$ &$1/30$ &$1/20$&$1/10$ &$\textbf{1}$\\
\hline
\textit{sp-en} EMD &$0.989$ &$0.969$ &$0.955$ &$0.931$ &$0.892$ &$0.820$\\
Time (s)&$3.4$ &$3.6$ &$3.7$ &$3.9$ &$5.3$ &$100.5$\\
\hline
\textit{it-en} EMD &$0.983$ &$0.966$ &$0.951$ &$0.935$ &$0.899$ &$0.847$ \\
                    Time (s)&$5.4$ &$5.9$ &$6.1$ &$6.5$ &$8.8$ &$162.6$ \\
\hline
\textit{ja-ch} EMD &$0.991$ &$0.975$ &$0.962$ &$0.941$ &$0.910$ &$0.836$ \\
          Time (s)&$1.6$ &$2.1$ &$2.2$ &$2.0$ &$3.0$ &$67.0$ \\
\hline
\textit{tu-en} EMD &$0.982$ &$0.960$ &$0.946$ &$0.922$ &$0.889$ &$0.839$ \\
                 Time (s)&$10.1$ &$10.4$ &$11.2$ &$11.9$ &$15.9$ &$287.0$ \\
\hline
\textit{ch-en} EMD &$1.014$ &$1.012$ &$0.990$ &$0.962$ &$0.923$ &$0.842$ \\
                    Time (s)&$1.9$ &$1.7$ &$2.2$ &$2.4$ &$2.7$ &$51.6$ \\
\hline
\end{tabular}
\label{tab-lin}
\end{table*}


Lemma~\ref{lem-tran} provides a complexity significantly lower than the previous $O(\lambda n d^2)$ (usually $n$ is much larger than $d$ in practice). 
After obtaining $\mathcal{T}(B)$, we can compute $\mathcal{EMD}(A, \mathcal{T}(B))$ in $\Gamma(n,d)$ time.
\textbf{Note} that the complexity $\Gamma(n,d)$ usually is $\Omega(n^2 d)$, which dominates the complexity of Step 2 and the second term $nd^2$ in the complexity by Lemma~\ref{lem-tran}. As a consequence, we have the following theorem for the running time. 

\begin{theorem}
\label{the-time}
Suppose $n=\max\{n_1, n_2\}\geq d$ and the algorithm of ~\cite{cohen1999earth} takes $\lambda\geq1$ rounds. The running time of Algorithm~\ref{alg-1} is 
$O\Big(\lambda\big(\Gamma(k, d)+k^2d+kd^2+d^3\big)\Big)+\Gamma(n, d)$.
\end{theorem}

If we run the same number of rounds on the original instance $(A, B)$ by the approach~\cite{cohen1999earth}, the total running time will be $O\Big(\lambda\big(\Gamma(n, d)+n^2d\big)\Big)$ by Proposition~\ref{pro-cg}. When $k\ll n$, Algorithm~\ref{alg-1} achieves a significant reduction on the running time.

\section{Experiments}
\label{sec-exp}

We implement our proposed algorithm and test the performance on both random and real datasets. All of the experimental results are obtained on a Windows workstation with 2.4GHz Intel Xeon CPU and 32GB DDR4 Memory. For each dataset, we run $20$ trials and report the average results. 
We set the iterative approach~\cite{cohen1999earth} to terminate when the change of the objective value is less than $10^{-3}$.

To construct a random dataset, we randomly generate two manifolds in $\mathbb{R}^{500}$, which are represented by the polynomial functions with low dimension ($\leq 50$); in the manifolds, we take two sets of randomly sampled points having the sizes of $n_1=2\times 10^4$ and $n_2=3\times 10^4$, respectively; also, for each sampled point, we randomly assign a positive weight; finally, we obtain two weighted point sets as an instance of geometric alignment. 

For real datasets, we consider the two applications mentioned in Section~\ref{sec-intro}, unsupervised bilingual lexicon induction and PPI network alignment. For the first application, we have 5 pairs of languages: {\em Chinese-English, Spanish-English, Italian-English, Japanese-Chinese, and Turkish-English}. Given by~\cite{DBLP:conf/emnlp/ZhangLLS17}, each language has a vocabulary list containing $3000$ to $13000$ words; we also follow their preprocessing idea to represent all the words by vectors in $\mathbb{R}^{50}$ through the embedding technique~\cite{mikolov2013exploiting}. Actually, each vocabulary list is represented by a distribution in the space where each vector has the weight equal to the corresponding frequency in the language. For the second application, we use the popular benchmark dataset NAPAbench~\cite{NAPAbench} of PPI networks. It contains $3$ pairs of PPI networks, where each network is a graph of $3000$-$10000$ nodes. As the step of  preprocessing, we apply the recent {\em node2vec} technique~\cite{grover2016node2vec} to represent each network by a group of vectors in $\mathbb{R}^{100}$; following~\cite{DBLP:conf/aaai/LiuDC017}, we assign a unit weight to each vector. 

\textbf{Results.} For each instance, we try different compression levels. According to Algorithm~\ref{alg-1}, we compress the size of each point set to be $k$. We set $k=\gamma\times \max\{n_1, n_2\}$ where $\gamma\in\{1/50, 1/40, 1/30, 1/20, 1/10, 1\}$; in particular, $\gamma=1$ indicates that we directly run the algorithm~\cite{cohen1999earth} without compression.
The purpose of our proposed approach is to design a compression method, such that the resulting qualities on the original and compressed datasets are close (as stated in Section~\ref{sec-prior} and the paragraph after the proof of Theorem~\ref{the-quality}). So in the experiment, we focus on the comparison with the results on the original data (i.e., the results of $\gamma=1$).

The results on random dataset are shown in Table~\ref{tab-random}. The obtained EMDs by compression are only slightly higher than the ones of $\gamma=1$, while the advantage of the compression on running time is significant. For example, the running time of $\gamma=1/50$ is less than $5\%$ of the one of $\gamma=1$. We obtain the similar performances on the real datasets. The results on Linguistic   are shown in Table~\ref{tab-lin} (due to the space limit, we put the results on PPI network dataset in the full version of our paper); the EMD for each compression level is always at most $1.2$ times the baseline with $\gamma=1$, but the corresponding running time is dramatically reduced. 

To further show the robustness of our method, we particularly add Gaussian noise to the random dataset and study the change of the objective value by varying the noise level. We set the standard variance of the Gaussian noise to be $\eta\times \Delta$, where $\Delta$ is the maximum diameter of the point sets and $\eta$ is from $0.5\times 10^{-2}$ to $2.5\times 10^{-2}$. Figure.~\ref{fig-robust} shows that the obtained EMD over baseline remains very stable (slightly higher than $1$) of each noise level $\eta$. 

\begin{figure}
\center
\includegraphics[height=1in]{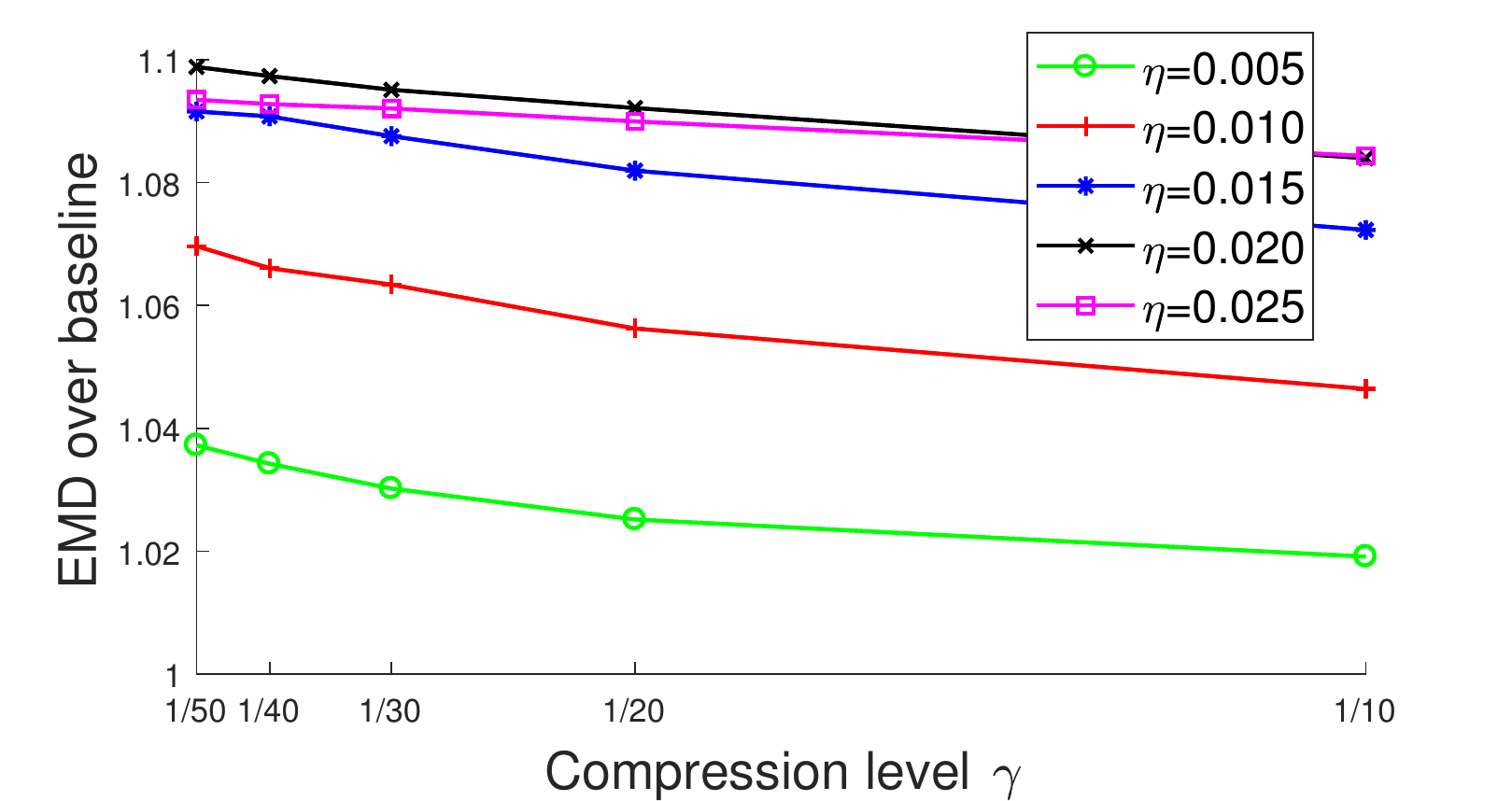}
    \caption{The EMDs over baseline for different noise levels.}
    \label{fig-robust}  
\end{figure}

\section{Conclusion}

In this paper, we propose a novel framework for compressing point sets in high dimension, so as to approximately preserve the quality of alignment. This work is motivated by several emerging applications in the fields of machine learning, bioinformatics, and wireless network. Our method utilizes the property of low doubling dimension, and yields a significant speedup on alignment. In the experiments on random and real datasets, we show that the proposed compression approach can efficiently reduce the running time to a great extent. 

\section{ Acknowledgments}
The research of this work was supported in part by NSF through grant CCF-1656905 and a start-up fund from Michigan State University. The authors also want to thank the anonymous reviewers for their helpful comments and suggestions for improving the paper.
\section{Appendix}
\begin{lemma}
\label{lem-app}
The multiplication $M_{A}\times M_{B}^{T}$ can be computed in $O(n_{1}n_{2}d+\min\{n_{1},n_{2}\}\cdot d^{2})$ time.
\end{lemma}
\begin{proof}
With a slight abuse of notations, we also use $F$ to denote the $n_1\times n_2$ matrix of the EMD flow where each entry is $f_{ij}$; also, $F_{i,:}$ represents the $i$-th row of the matrix $F$. Given a vector $t$, we use $\sqrt{t}$ to denote the new vector with each entry being the square root of the corresponding one in $t$. Also, we use $diag(t)$ to denote the diagonal matrix where the $i$-th diagonal entry is the $i$-th entry of $t$. Following the constructions of $M_A$ and $M_B$, we have 
\begin{eqnarray}
M_{A}&=&[\sqrt{f_{11}}a_{1}^{1},\cdots,\sqrt{f_{1n_{2}}}a_{1}^{n_{2}};\nonumber\\
&&\cdots\cdots\cdots\nonumber\\
&&\sqrt{f_{n_{1}1}}a_{n_{1}}^{1},\cdots,\sqrt{f_{n_{1}n_{2}}}a_{n_{1}}^{n_{2}}]\nonumber\\
&=&[a_{1}\sqrt{F_{1,:}},a_{2}\sqrt{F_{2,:}},\cdots,a_{n_{1}}\sqrt{F_{n_{1},:}}];\nonumber\\
M_{B}&=&[\sqrt{f_{11}}b_{1}^{1},\cdots,\sqrt{f_{1n_{2}}}b_{n_{2}}^{1};\nonumber\\
&&\cdots\cdots\cdots\nonumber\\
&&\sqrt{f_{n_{1}1}}b_{1}^{n_{1}},\cdots,\sqrt{f_{n_{1}n_{2}}}b_{n_{2}}^{n_{1}}]\nonumber\\
&=&B[\textnormal{diag}(\sqrt{F_{1,:}}),\textnormal{diag}(\sqrt{F_{2,:}}),\nonumber\\
&&\cdots,\textnormal{diag}(\sqrt{F_{n_{1},:}})],\nonumber
\end{eqnarray}
by some simple calculation. 
Then,
\begin{align*}
M_{A}\times M_{B}^{T}&=\sum_{i=1}^{n_{1}}\big(a_{i}\sqrt{F_{i,:}}\big)\times\big(\textnormal{diag}(\sqrt{F_{i,:}})B^{T}\big)\\
&=\sum_{i=1}^{n_{1}}a_{i}F_{i,:}B^{T}=AFB^{T}.
\end{align*}
It is easy to know that computing $AFB^{T}$ takes $O(n_{1}n_{2}d+\min\{n_{1},n_{2}\}\cdot d^{2})$ time.
\end{proof}

\bibliographystyle{aaai}
\bibliography{nips_2017}

\end{document}